%% file: fourier_icml16.tex
\documentclass{article}

\usepackage{times}
\usepackage{graphicx} 

\usepackage{natbib}

\usepackage{algorithm}
\usepackage[noend]{algorithmic}

\usepackage{caption}
\usepackage{subcaption}

\usepackage{hyperref}

\usepackage{amsmath}
\usepackage{amssymb}
\usepackage{amsthm}

\usepackage[accepted]{icml2016} 

\icmltitlerunning{Variable Elimination in the Fourier Domain}

\begin{document} 

\twocolumn[
\icmltitle{Variable Elimination in the Fourier Domain}

\icmlauthor{Yexiang Xue}{yexiang@cs.cornell.edu}
\icmladdress{Cornell University, Ithaca, NY, 14853, USA}
\icmlauthor{Stefano Ermon}{ermon@cs.stanford.edu}
\icmladdress{Stanford University, Stanford, CA, 94305, USA}
\icmlauthor{Ronan Le Bras, Carla P. Gomes, Bart Selman}{\{lebras,gomes,selman\}@cs.cornell.edu}
\icmladdress{Cornell University, Ithaca, NY, 14853, USA}

\icmlkeywords{Fourier Representation, Variable Elimination, Probabilistic Inference, Graphical Models}

\vskip 0.3in
]

\def\edkl{{\sc ed-kl}}
\def\edbp{{\sc ed-bp}}

\def\clone(#1){\hat{#1}}
\def\edse(#1){{{\it SE}(#1)}}
\def\edpm(#1){{{\it PM}(\clone(#1))}}


\def\real{{\mathbb{R}}}

\def\gm{{\mathcal{M}}}
\def\gmp{{\mathcal{M}^\prime}}
\def\gme{{\mathcal{E}}}
\def\gmep{{\mathcal{E}^\prime}}

\def\bn{{\mathcal{N}}}
\def\bnp{{\mathcal{N}^\prime}}

\def\n(#1){\bar{#1}}
\def\u{{\bf u}}
\def\U{{\bf U}}
\def\v{{\bf v}}
\def\V{{\bf V}}
\def\x{{\bf x}}
\def\X{{\bf X}}
\def\y{{\bf y}}
\def\Y{{\bf Y}}
\def\z{{\bf z}}
\def\Z{{\bf Z}}
\def\e{{\bf e}}
\def\E{{\bf E}}
\def\ep{{\bf e^\prime}}
\def\Ep{{\bf E^\prime}}

\def\pr{{\it Pr}}
\def\prp{{\it Pr^\prime}}
\def\prh{\widehat{\it Pr}}

\def\bel{{\it BEL}}
\def\mi{{\it MI}}
\def\ent{{\it ENT}}
\def\kl{{\it KL}}
\def\ind{{\it IND}}

\def\Normal{{\mathcal{N}}}
\def\mean{{\mathbb{E}}}

\newcommand\name[1]{\ensuremath{\mathsf{#1}}}
\def\true{\name{true}}
\def\false{\name{false}}


\def\NP{{\mathrm{NP}}}
\def\PP{{\mathrm{PP}}}
\def\NPPP{{\NP^{\PP}}}

\newtheorem{theorem}{Theorem}
\newtheorem{corollary}{Corollary}
\newtheorem{lemma}{Lemma}
\newtheorem{definition}{Definition}
\newtheorem{proposition}{Proposition}
\newtheorem{claim}{Claim}
\newtheorem{condition}{Condition}
\newtheorem{conjecture}{Conjecture}

\begin{abstract} 
\input{abstract}

\end{abstract}

\input{introduction}

\input{preliminary}

\input{theory}

\input{message_passing}
\input{experiments}

\input{conclusion}

{
\bibliography{BayesianOpt2,BooleanFunction,message_passing_parity_constraint}
\bibliographystyle{icml2016}
}
\clearpage

\newpage
\input{theory_supply}

\end{document}

%% file: abstract.tex
The ability to represent complex high dimensional probability
distributions in a compact form is one of the key insights 
in the field of graphical models. Factored representations are ubiquitous in machine learning and lead to
major computational advantages. 
We explore a different type of compact
representation based on \textit{discrete Fourier representations}, complementing the classical approach based on conditional independencies. 
%
%
We show that a large class of probabilistic graphical models have a compact
Fourier representation. This theoretical result opens up an entirely
new way of approximating a probability distribution.  We demonstrate
the significance of this approach by applying it to the variable elimination
algorithm.  Compared with the traditional bucket representation and
other approximate inference algorithms, we obtain significant improvements.


%% file: introduction.tex

\section{Introduction}


Probabilistic inference is a key computational challenge in
statistical machine learning and artificial intelligence. Inference
methods have a wide range of applications, from learning models to
making predictions and informing decision-making using statistical
models. Unfortunately, the inference problem is computationally
intractable, and standard exact inference algorithms, such as variable
elimination and junction tree algorithms have worst-case exponential
complexity.
\begin{figure}[t]
  \centering
  \includegraphics[width=0.7\linewidth]{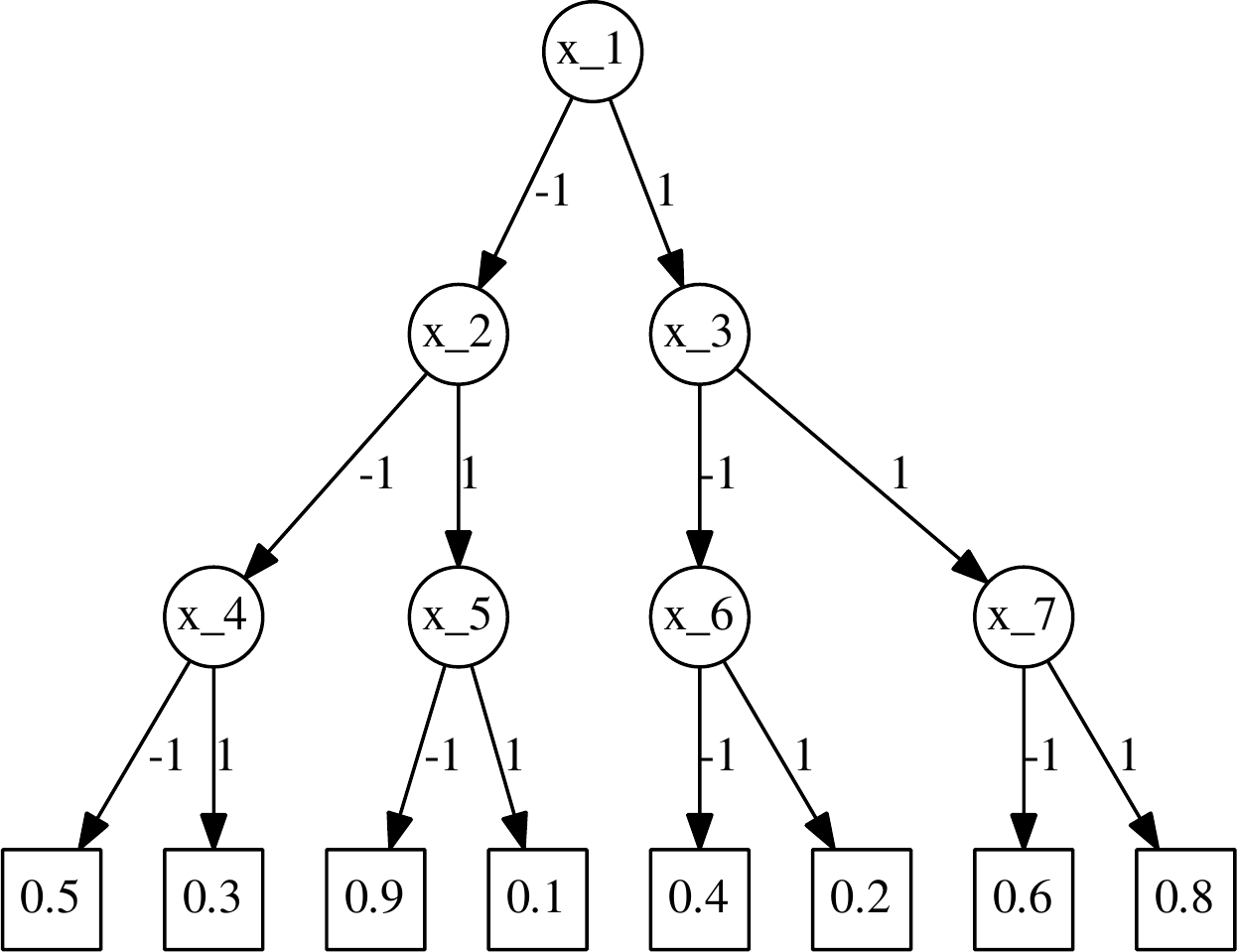}
  \caption{An example of a decision tree representing a function $f:
    \{x_1, \dots, x_7\} \rightarrow \mathcal{R}^+$.}
  \label{fig:decisiontree}
  \vskip -0.5cm
\end{figure}

The ability to represent complex high dimensional probability distributions in a compact form is perhaps the most important insight in the field of graphical models. 
The fundamental idea is to exploit (conditional) independencies between the
variables to achieve compact \emph{factored} representations, where a complex global model is represented as a product of simpler, local models. 
Similar ideas have been considered in the analysis of Boolean functions and logical forms \cite{Dechter97MiniBucket}, as well as in physics with low rank tensor decompositions and matrix product states representations \cite{Jordan1999Variational,Linden2003CollaborativeFiltering,Sontag08tightenLP,Friesen2015RecursiveDecomp}. 

Compact representations are also key for the development of efficient inference algorithms, including message-passing ones.
Efficient algorithms can be developed when messages representing the interaction among many
variables can be decomposed or approximated with the product of several smaller messages, each
involving a subset of the original variables. 
Numerous approximate and exact inference algorithms are based on this idea \cite{1993ADD,Flerova2011MinibucketMomentMatching,MateescuKGD10,GogateD13StructureBP,WainwrightJW03,Darwiche2002compilation,IhlerFDO12,HazanJ12}.  

Conditional independence (and related factorizations) is not the only type of structure that can be exploited to achieve compactness. For example, consider the weighted decision tree in
Figure~\ref{fig:decisiontree}. No two variables
in the probability distribution in
Figure~\ref{fig:decisiontree} are independent of each other. The probability distribution cannot be represented by the product of simpler terms of disjoint domains and hence we cannot take advantage of independencies. The full probability table needs $2^7=128$ entries to be
represented exactly.
Nevertheless, this table can be described exactly by 8 simple decision
rules, each corresponding to a path from the root to a leaf in the tree.


In this paper, we explore a novel way to exploit compact representations of high-dimensional probability tables in (approximate) probabilistic inference algorithms.
Our approach is based on a (discrete) Fourier representation of the tables, which can be interpreted as a change of basis. Crucially, tables that are dense in the canonical basis can have a sparse Fourier representation. In particular, under certain conditions, probability tables can be represented (or well approximated) using a small number of Fourier coefficients.
%
%
The Fourier representation has found numerous recent applications, including modeling stochastic processes \cite{rogers2000evaluating,abbring2012likelihood}, manifolds \cite{cohen2015harmonicmanifold}, and permutations \cite{HuangGG09FourierPermutation}. 
Our approach is based on Fourier representation on Boolean functions, which has found tremendous success in PAC learning \cite{O'Donnell2008,Mansour1994,Blum1998,BuchmanSMPF12FourierLearning}, but these ideas have not been fully exploited in the fields of probabilistic inference and graphical models.


In general, a factor over $n$ Boolean variables requires
$O(2^n)$ entries to be specified, and similarly the corresponding
Fourier representation is dense in general, i.e., it has $O(2^n)$
non-zero coefficients.  However, a rather surprising fact which was
first discovered by Linial \cite{Linial1993LowOrder} is that factors
corresponding to fairly general classes of logical forms admit a
compact Fourier representation.  Linial discovered that formulas in
Conjunctive Normal Form (CNF) and Disjunctive Normal Form (DNF) with
bounded width (the number of variables in each clause) have compact
Fourier representations. 
%

In this paper, we introduce a novel approach for using approximate
Fourier representations in the field of probabilistic inference.
We generalize the work of Linial to the case of probability distributions (the weighted case where the entries are not necessarily 0 or 1), showing
that a large class of probabilistic graphical models have compact
Fourier representation. 
%
The proof extends the Hastad's Switching Lemma~\cite{Hastad1987}
to the weighted case.
At a high level, a compact Fourier representation often means the
weighted probabilistic distribution can be captured by a small set of
critical decision rules. Hence, this notion is closely related to
decision trees with bounded depth.


Sparse (low-degree) Fourier representations provide an entirely new way of
approximating a probability distribution.
We demonstrate the power of this idea by applying it to the variable
elimination algorithm.
%
Despite that it is conceptually simple, we show in
Table~\ref{tab:uai2010} that the variable elimination algorithm with
Fourier representation outperforms Minibucket, Belief Propagation and
MCMC, and is competitive and even outperforms an award winning solver 
HAK on several categories of the UAI Inference Challenge.



%% file: preliminary.tex

\section{Preliminaries}

\subsection{Inference in Graphical Models}

We consider a Boolean graphical model over $N$ Boolean variables
$\{x_1, x_2, \ldots, x_N\}$. 
We use bold typed variables to represent a vector of variables.  For
example, the vector of all Boolean variables $\x$ is written as $\x =
(x_1, x_2, \ldots, x_N)^T$.
We also use $\x_S$ to represent the image of vector $\x$
\textit{projected} onto a subset of variables: $\x_S = (x_{i_1},
x_{i_2}, \ldots, x_{i_k})^T$ where $S = \{{i_1}, \ldots, {i_k}\}$. 
A probabilistic graphical model is defined as:
$$Pr(\x) = \frac{1}{Z} f(\x) = \frac{1}{Z} \prod_{i=1}^K \psi_i(\x_{S_i}).$$
where each $\psi_i: \{-1,1\}^{|S_i|} \rightarrow \real^+$ is called a
\textit{factor}, and is a function that depends on a subset of
variables whose indices are in ${S_i}$.
$Z = \sum_{\x} \prod_{i=1}^K \psi_i(\x_{S_i})$ is the normalization
factor, and is often called the \textit{partition function}.
In this paper, we will use $-1$ and $1$ to represent $\false$ 
and $\true$. 
We consider two key probabilistic inference tasks: the computation of the partition function $Z$ (\texttt{PR}) and marginal probabilities  $Pr(e) = \frac{1}{Z} \sum_{\x \sim e} f(\x)$
(\texttt{Marginal}), in which $\x \sim e$ means that $\x$ is consistent with the evidence $e$. 

%

The Variable Elimination Algorithm is an exact algorithm to compute
marginals and the partition function for general graphical models. It starts with a variable ordering
$\pi$. In each iteration, it eliminates one variable by multiplying all
factors involving that variable, and then summing that variable out. When all
variables are eliminated, the factor remaining is a singleton, whose
value corresponds to the partition function.
The complexity of the VE algorithm depends on the size of the largest factors
generated during the elimination process, and is known to be
exponential in the tree-width~\cite{Gogate2004Treewidth}.

Detcher proposed the Mini-bucket Elimination Algorithm
\cite{Dechter97MiniBucket}, which dynamically decomposes and
approximates factors (when the domain of a product exceeds a
threshold) with the product of smaller factors during the elimination
process. Mini-bucket can provide upper and lower bounds on the
partition function.
The authors of
\cite{Rooij09FastSubsetConvolution,Smith2013InclusionExclusion}
develop fast operations similar to the Fast Fourier transformation,
and use it to speed up the exact inference. Their approaches do not
approximate the probability distribution, which is a key difference from
 this paper.

\subsection{Hadamard-Fourier Transformation}

  

Hadamard-Fourier transformation has attracted a lot of attention in
PAC Learning Theory. 
%
Table~\ref{tab:small_example} provides an example where a
function $\phi(x, y)$ is transformed into its Fourier representation.
The transformation works by writing $\phi(x,y)$ using
interpolation, then re-arranging the terms to get a canonical term.
The example can be generalized, and it can be shown that any function defined on a Boolean hypercube has an equivalent Fourier representation. 
\begin{theorem}
  (Hadamard-Fourier Transformation) Every $f: \{-1, 1\}^n \rightarrow \real$ can be uniquely expressed as a multilinear polynomial,
  $$f(\x) = \sum_{S \subseteq [n]} c_S \prod_{i \in S} x_i.$$
  where each $c_S \in \real$. This polynomial is referred to as the Hadamard-Fourier expansion of $f$.
\end{theorem}
Here, $[n]$ is the power set of $\{1,\ldots,n\}$. 
Following standard notation, we will write $\hat{f}(S)$ to denote the coefficient $c_S$ and $\chi_S(\x)$ for the basis function $\prod_{i \in S} x_i$.  As a special case, $\chi_\emptyset = 1$. Notice these basis functions are parity functions. We also call $\hat{f}(S)$ a degree-$k$ coefficient of $f$ iff $|S|=k$. 
In our example in Table~\ref{tab:small_example}, the coefficient for basis function $xy$ is $\hat{\phi}(\{x, y\}) = \frac{1}{4} (\phi_1 - \phi_2 - \phi_3 + \phi_4)$, which is a degree-2 coefficient.

We re-iterate some classical results on Fourier expansion.  
First, as with the classical (inverse) Fast Fourier Transformation
(FFT) in the continuous domain, there are similar divide-and-conquer
algorithms ($\mathtt{FFT}$ and $\mathtt{invFFT}$) which connect the
table representation of $f$ (e.g., upper left table,
Table~\ref{tab:small_example}) with its Fourier representation (e.g.,
bottom representation, Table~\ref{tab:small_example}).
Both $\mathtt{FFT}$ and $\mathtt{invFFT}$ run in time $O(n \cdot 2^n)$
for a function involving $n$ variables.
In fact, the length $2^n$ vector of all function values and the length
$2^n$ vector of Fourier coefficients are connected by a $2^n$-by-$2^n$
matrix $H_n$, which is often called the $n$-th Hadamard-Fourier matrix.
%
In addition, we have the Parseval's identity for Boolean Functions as well: $\mathbb{E}_{\x}[f(\x)^2] = \sum_S \hat{f}(S)^2$.
  



\begin{table}[t]
  \centering
  \begin{minipage}{0.4\linewidth}
    \centering
    \begin{tabular}{cc|c}
      \hline
      $x$ & $y$ & $\phi(x, y)$\\
      \hline
      -1 & -1 & $\phi_1$\\
      -1 & 1 & $\phi_2$\\
      1 & -1 & $\phi_3$\\
      1 & 1 & $\phi_4$\\
      \hline
    \end{tabular}
  \end{minipage}
  \begin{minipage}{0.55\linewidth}
    \begin{align*}
                 \phi(x, y) =
                 & \frac{1-x}{2} \cdot \frac{1-y}{2} \cdot \phi_1 +\\
                 &  \frac{1-x}{2} \cdot \frac{1+y}{2} \cdot \phi_2+ \\
                 &  \frac{1+x}{2} \cdot \frac{1-y}{2} \cdot \phi_3 +\\
                 &  \frac{1+x}{2} \cdot \frac{1+y}{2} \cdot \phi_4.
    \end{align*}
  \end{minipage}  
  \vspace{-0.1in}
  \begin{minipage}{0.8\linewidth}
    \begin{align*}
                 &\phi(x, y) =\\
      &\frac{1}{4} (\phi_1 + \phi_2 + \phi_3 + \phi_4) + \frac{1}{4} (-\phi_1 - \phi_2 + \phi_3 + \phi_4)x\\ 
                 &+\frac{1}{4} (-\phi_1 + \phi_2 - \phi_3 + \phi_4)y + \frac{1}{4} (\phi_1 - \phi_2 - \phi_3 + \phi_4)xy.
    \end{align*}
  \end{minipage}  
  \caption{(Upper Left) Function $\phi:\{-1, 1\}^2 \rightarrow \real$ is represented in a table. (Upper Right) $\phi$ is re-written using interpolation. (Bottom) The terms of the upper-right equation are re-arranged, which yields the Fourier expansion of function $\phi$.}
  \label{tab:small_example}
\end{table}

%% file: theory.tex

\section{Low Degree Concentration of Fourier Coefficients}
\label{sec:theory}

Fourier expansion replaces the table
representation of a weighted function with its Fourier
coefficients.
For a function with $n$ Boolean variables, the complete table
representation requires $2^n$ entries, and so does the full Fourier
expansion.
Interestingly, many natural
functions can be approximated well with only a few Fourier 
coefficients. 
This raises a natural question: 
\emph{what type of functions can be well
approximated with a compact Fourier expansion?}

We first discuss which functions can be represented \textit{exactly} in
the Fourier domain with coefficients up to degree $d$.
To answer this question, we show a tight connection between Fourier
representations with bounded degree and decision trees with bounded
depth.
%
%
A decision tree for a weighted function $f: \{-1, 1\}^n \rightarrow
\real$ is a tree in which each inner node is labelled with one
variable, and has two out-going edges, one labelled with $-1$, and other
one with $1$. The leaf nodes are labelled with real
values.
When evaluating the value on an input $\x = x_1 x_2 \ldots x_n$, we start
from the root node, and follow the corresponding out-going edges by
inspecting the value of one variable at each step, until we reach one
of the leaf nodes. The value at the leaf node is the output for
$f(\x)$.
The \textit{depth} of the decision tree is defined as the longest path
from the root node to one of the leaf nodes.
Figure~\ref{fig:decisiontree} provides a decision tree representation
for a weighted Boolean function.
One classical result \cite{O'Donnell2008} states that if a
function can be captured by a decision tree with depth $d$, then it
can be represented with Fourier coefficients up to degree $d$:
\begin{theorem}
  Suppose $f: \{-1, 1\}^n \rightarrow \real$ can be
  represented by a decision tree of depth $d$, then all the
  coefficients whose degree are larger than $d$ is zero in $f$'s Fourier
  expansion: $\hat{f}(S)=0$ for all $S$ such that $|S| > d$.
  \label{th:decisiontree}
\end{theorem}

We can also provide the converse of Theorem~\ref{th:decisiontree}:
\begin{theorem}
  Suppose $f: \{-1, 1\}^n \rightarrow \real$ can be
  represented by a Fourier expansion with non-zero coefficients up to degree $d$, then $f$ can be represented by the sum of several decision trees, each of which has depth at most $d$.
  \label{th:decisiontreeinverse}
\end{theorem}
Theorem~\ref{th:decisiontree} and
Theorem~\ref{th:decisiontreeinverse} provide a tight connection
between the Fourier expansion and the decision trees. 
This is also part of the reason why the Fourier representation is a
powerful tool in PAC learning.
%
%
Notice that the Fourier representation complements the classical way
of approximating weighted functions exploiting
independencies.
%
To see this, suppose there is a decision tree of the same structure as
in Figure~\ref{fig:decisiontree}, but has 
depth $d$. According to Theorem~\ref{th:decisiontree}, it can be
represented exactly with Fourier coefficients up to degree $d$.
In this specific example, the number of non-zero Fourier coefficients is
$O(2^{2d})$.
Nonetheless, no two variables in figure~\ref{fig:decisiontree} are
independent with each other.  Therefore, it's not possible to
decompose this factor into a product of smaller factors with disjoint
domains (exploiting independencies). Notice that the full table
representation of this factor has $O(2^{2^d})$ entries, because
different nodes in the decision tree have different variables and there
are $O(2^{d})$ variables in total in this example.

If we are willing to accept an approximate representation, low degree 
Fourier coefficients can capture an even wider class of functions. 
%
%
We follow the standard notion of $\epsilon$-concentration:
\begin{definition}
  The Fourier spectrum of $f: \{ -1,1\}^n \rightarrow \real$ is
  $\epsilon$-concentrated on degree up to $k$ if and only if
  $\mathcal{W}_{>k}[f] = \sum_{S \subseteq [n], |S| > k} \hat{f}(S)^2
  < \epsilon$.
\end{definition}

We say a CNF (DNF) formula has  bounded width $w$ if and only
if every clause (term) of the CNF (DNF) has at most $w$ literals. 
In the literatures outside of PAC Learning, this is also referred to as
a CNF (DNF) with clause (term) length $w$.
Linial \cite{Linial1993LowOrder} proved the following result: 
\begin{theorem}
  [Linial] Suppose $f: \{-1,1\}^n \rightarrow \{-1,1\}$ is computable
  by a DNF (or CNF) of width $w$, then $f$'s Fourier spectrum is
  $\epsilon$-concentrated on degree up to $O(w \log(1/\epsilon))$.
\end{theorem}

Linial's result demonstrates the power of Fourier representations,
since bounded width CNF's (or DNF's) include a very rich class of
functions.
Interestingly, the bound does not depend on the number of clauses,
even though the clause-variable ratio is believed to characterize the
hardness of satisfiability problems.


As a contribution of this paper, we extend Linial's results to a
class of weighted probabilistic graphical models, which are
contractive with gap $1-\eta$ and have bounded width $w$.
To our knowledge, this extension from the
deterministic case to the probabilistic case is novel.
\begin{definition}
Suppose $f(\x): \{-1, 1\}^n \rightarrow \real^+$ is a weighted
function, we say $f(\x)$ has bounded width $w$ iff the number of variables in the domain of $f$ is no more than $w$. We say
$f(\x)$ is \textit{contractive} with gap $1-\eta$ $(0\leq \eta < 1)$ if and only
if (1) for all $\x$, $f(\x) \leq 1$; (2) $\max_{\x} f(\x) = 1$; (3)
if $f(\x_0) < 1$, then $f(\x_0) \leq \eta$. 
\end{definition}

The first and second conditions are mild restrictions.  For a
graphical model, we can always rescale each factor properly to ensure
its range is within $[0, 1]$ and the largest element is 1. The
approximation bound we are going to prove depends on the gap
$1-\eta$. 
Ideally, we want $\eta$ to be small. The class of contractive
functions with gap $1-\eta$ still captures a wide class of interesting
graphical models. For example, it captures Markov Logic Networks
\cite{Richardson2006MLN}, when the weight of each clause is large.
Notice that this is one of the possible necessary conditions we
found success in proving the weight concentration result. 
In practice, because compact Fourier representation is more about the
structure of the weighted distribution (captured by a series of
decision trees of given depth), graphical models with large $\eta$
could also have concentrated weights.
%
The main theorem we are going to prove is as follows:
\begin{theorem}
  (Main) Suppose $f(\x) = \prod_{i=1}^m f_i(\x_i)$, in which every $f_i$ is a
  contractive function with width $w$ and gap $1-\eta$,
  then $f$'s Fourier spectrum is
  $\epsilon$-concentrated on degree up to $O(w \log(1/\epsilon) \log_\eta \epsilon)$ when $\eta > 0$ and $O(w \log(1/\epsilon))$ when $\eta = 0$.
  \label{th:sparse_main}
\end{theorem}

The proof of theorem~\ref{th:sparse_main} relies on the notion of
random restriction and our own extension to the Hastad's Switching
Lemma~\cite{Hastad1987}.
\begin{definition}
Let $f(\x): \{-1, 1\}^n \rightarrow \real$ and $J$ be
subset of all the variables $x_1,\ldots,x_n$. Let $\z$ be an
assignment to remaining variables $\overline{J} = \{-1, 1\}^n
\setminus J$. Define $f|_{J|\z}: \{-1, 1\}^J \rightarrow \real$
to be the restricted function of $f$ on $J$ by setting all the remaining
variables in $\overline{J}$ according to $\z$.
\end{definition}
\begin{definition}
  ($\delta$-random restriction) A $\delta$-random restriction of
  $f(\x): \{-1, 1\}^n \rightarrow \real$ is defined as $f|_{J|\z}$,
  when elements in $J$ are selected randomly with probability
  $\delta$, and $\z$ is formed by randomly setting variables in
  $\overline{J}$ to either $-1$ or $1$.
  We also say $J|\z$ is a $\delta$-random restriction set.
  \label{def:randomres}
\end{definition}
With these definitions, we proved our weighted extension to
the Hastad's Switching Lemma:
\begin{lemma}
  (Weighted Hastad's Switching Lemma) 
  Suppose $f(\x) = \prod_{i=1}^m f_i(\x_i)$, in which every $f_i$ is
  a contractive function with width $w$ and gap $1-\eta$. 
  Suppose $J|\z$ is a $\delta$-random restriction set, then
  \begin{align}
  & Pr\left(\exists \mbox{ decision tree } h \mbox{ with depth } t, \ ||h-f_{J|\z}||_\infty \leq \gamma \right) \nonumber\\
  & \geq 1- \frac{1}{2}\left(\frac{\delta}{1-\delta} 8 u w\right)^t.\nonumber
  \end{align}
  in which $u=\lceil \log_{\eta} \gamma \rceil+1$ if $0 < \eta < 1$ or $u=1$  if $\eta = 0$ and $||.||_\infty$ means $\max|.|$.
  \label{lem:switching}
\end{lemma}
The formal proof of Lemma~\ref{lem:switching} is based on a clever
generalization of the proof by Razborov
for the unweighted case~\cite{Razborov1995Hastad}, and is deferred to the supplementary
materials.

\begin{lemma}
  Suppose $f(\x): \{-1, 1\}^n \rightarrow \real$ and $|f(\x)| \leq 1$. $J|\z$ is a
  $\delta$-random restriction set. $t\in \mathbb{N}$, $\gamma > 0$ and
  let $\epsilon_0 = Pr\{\neg \exists $ decision tree  $h$ with depth  $t$  such that  $||f|_{J|\z} - h||_\infty \leq \gamma\}$,
  then the Fourier spectrum of $f$ is $4\left(\epsilon_0 + (1-\epsilon_0) \gamma^2 \right)$-concentrated on degree up to $2t/\delta$.
\end{lemma}
\begin{proof}
  We first bound $\mean_{J|\z} \left[\sum_{S\subseteq[n], |S| > t}
    \hat{f}|_{J|\z}(S)^2 \right]$. With probability $1-\epsilon_0$, there is
  a decision tree $h$ with depth $t$ such that $||f|_{J|\z}(\x) -
  h(\x)||_\infty \leq \gamma$.
  %
  In this scenario, 
  \begin{equation}
  \sum_{S\subseteq[n], |S| > t} \hat{f}|_{J|\z}(S)^2 = \sum_{S\subseteq[n], |S| > t} \left(\hat{f}|_{J|\z}(S) - \hat{h}(S)\right)^2.
  \label{eq:lem2_proof1}
  \end{equation}
  This is because due to Theorem~\ref{th:decisiontree}, $\hat{h}(S)=0$
  for all $S$ such that $|S| > t$.
  Because $|f|_{J|\z}(\x) - h(\x)| \leq \gamma$ for all $\x$, hence
  the right side of Equation~\ref{eq:lem2_proof1} must satisfy
  \begin{align}
    &\sum_{S\subseteq[n],|S| > t} \left(\hat{f}|_{J|\z}(S) - \hat{h}(S)\right)^2 
    \leq \sum_{S\subseteq[n]} \left(\hat{f}|_{J|\z}(S) - \hat{h}(S)\right)^2\nonumber\\
    &= \mean \left[ \left(f|_{J|\z}(\x) - h(\x)\right)^2 \right]
    \leq \gamma^2.\label{eq:lem2_proof2}
  \end{align}
  The second to the last equality of Equation~\ref{eq:lem2_proof2} is
  due to the Parseval's Identity.
  With probability $\epsilon_0$, there are no decision trees close to
  $f|_{J|\z}$. However, because $|f|_{J|\z}| \leq 1$, we must have 
  $\sum_{S\subseteq[n], |S| > t} \hat{f}|_{J|\z}(S)^2 \leq 1$.
  %
  Summarizing these two points, we have:
  $$\mean_{J|\z} \left[\sum_{S\subseteq[n], |S| > t}
    \hat{f}|_{J|\z}(S)^2 \right] \leq (1-\epsilon_0) \gamma^2 + \epsilon_0.$$

  Using a known result $\mean_{J|\z}\left[\hat{f}|_{J|\z}(S)^2\right]= \sum_{U \subseteq [n]}Pr\{U \cap J = S\} \cdot \hat{f}(U)^2$, we have: 
  \begin{align}
    &\mean_{J|\z} \left[\sum_{S\subseteq[n], |S| > t} \hat{f}|_{J|\z}(S)^2 \right]\nonumber
   =\sum_{S\subseteq[n], |S| > t} \mean_{J|\z} \left[\hat{f}|_{J|\z}(S)^2 \right]\nonumber\\
   &=\sum_{U\subseteq[n]} Pr\{|U \cap J| > t\} \cdot \hat{f}(U)^2.
  \end{align}  
  The distribution of random variable $|U \cap J|$ is Binomial($|U|,\delta$). When $|U| \geq 2t/\delta$, this variable has mean at least $2t$, using Chernoff bound, $Pr\{|U \cap J| \leq t\} \leq \left( 2/e \right)^t < 3/4$. Therefore,
  \begin{align}
    (1-\epsilon_0) \gamma^2 + \epsilon_0\nonumber
    \geq& \sum_{U\subseteq[n]} Pr\{|U \cap J| > t\} \cdot \hat{f}(U)^2\nonumber\\
    \geq& \sum_{U\subseteq[n], |U| \geq 2t/\delta} Pr\{|U \cap J| > t\} \cdot \hat{f}(U)^2\nonumber\\
    \geq& \sum_{U\subseteq[n], |U| \geq 2t/\delta} \left(1-\frac{3}{4}\right) \cdot \hat{f}(U)^2\nonumber.
  \end{align}
  
  We get our claim $\sum_{|U| \geq 2t/\delta} \hat{f}(U)^2 \leq 4((1-\epsilon_0) \gamma^2 + \epsilon_0)$.
\end{proof}

  Now we are ready to prove Theorem~\ref{th:sparse_main}. 
  Firstly suppose $\eta > 0$, choose $\gamma = \sqrt{{\epsilon/8}}$,
  which ensures $4(1-\epsilon_0)\gamma^2 \leq 1/2\cdot \epsilon$.
  Next choose $\delta = 1/(16uw + 1)$, $t=C \log(1/\epsilon)$,
  which ensures
  $$\epsilon_0 = \frac{1}{2}\left(\frac{\delta}{1-\delta} 8 u w\right)^t = \frac{1}{2} \epsilon^C.$$
  Choose $C$ large enough, such that $4 \cdot 1/2 \cdot \epsilon^C \leq 1/2\cdot \epsilon$. 
  Now we have $4((1-\epsilon_0) \gamma^2 + \epsilon_0) \leq \epsilon$.
  At the same time, $2t/\delta = C \log(1/\epsilon) (16uw + 1) = O(w \log(1/\epsilon) \log_\eta \epsilon)$.\footnote{$\eta=0$ corresponds  to the classical CNF (or DNF) case.}


%% file: message_passing.tex

\section{Variable Elimination in the Fourier Domain}

We have seen above that a Fourier representation can provide a useful
compact representation of certain complex probability distributions.
In particular, this is the case for distributions that can be captured
with a relatively sparse set of Fourier coefficients.  We will now
show the practical impact of this new representation by using it in
an inference setting.
In this section, we propose an
inference algorithm which works like the classic Variable Elimination
(VE) Algorithm, except for passing messages represented in the Fourier
domain.

The classical VE algorithm consists of two basic steps --
the multiplication step and the elimination step.  The multiplication
step takes $f$ and $g$, and returns $f \cdot g$, while the elimination
step sums out one variable $x_i$ from $f$ by returning $\sum_{x_i} f$.
Hence, the success of the VE procedure in the Fourier domain depends on
efficient algorithms to carry out the aforementioned two steps.
A naive approach is to transform the representation back to the value
domain, carry out the two steps there, then transform it
back to Fourier space. While correct, this strategy would eliminate all the benefits of Fourier
representations.


%
%
Luckily, the elimination step can be carried out in the Fourier domain as follows: 
\begin{theorem}
  Suppose $f$ has a Fourier expansion: $f(\x) = \sum_{S \subseteq [n]}
  \hat{f}(S) \chi_S(\x)$. Then the Fourier expansion for $f' =
  \sum_{x_i} f$ when $x_i$ is summed out is: $\sum_{S \subseteq [n]}
  \hat{f}'(S) \chi_S(\x)$, where $\hat{f}'(S) = 2\hat{f}(S)$ if $i
  \not \in S$ and $\hat{f}'(S) = 0$ if $i \in S$.
  \label{th:ve_sum}
\end{theorem}

The proof is left to the supplementary materials. 
From Theorem~\ref{th:ve_sum}, one only needs a linear scan of all the
Fourier coefficients of $f$ in order to compute the Fourier expansion
for $\sum_{x_0} f$. Suppose $f$ has $m$ non-zero coefficients in its
Fourier representation, this linear scan takes time $O(m)$.

There are several ways to implement the multiplication step. The first
option is to use the school book multiplication. To multiply functions $f$
and $g$, one multiplies every pair of their Fourier coefficients, and
then combines similar terms. If $f$ and $g$ have $m_f$ and $m_g$ terms
in their Fourier representations respectively, this operation takes
time $O(m_fm_g)$.
%
As a second option for multiplication, one can convert $f$ and $g$ to
their value domain, multiply corresponding entries, and then convert
the result back to the Fourier domain. Suppose the union of the
domains of $f$ and $g$ has $n$ variables ($2^n$ Fourier terms), the
conversion between the two domains dominates the complexity, which is
$O(n \cdot 2^n)$.
Nonetheless, when $f$ and $g$ are relatively dense, this method could
have a better time complexity than the school book multiplication.
In our implementation, we trade the complexity between the
aforementioned two options, and always use the one with lower time
complexity.

Because we are working on models in which exact inference is
intractable, sometimes we need to truncate the Fourier representation to
prevent an exponential explosion.
We implement two variants for truncation. 
One is to keep low degree Fourier coefficients, which is inspired by
our theoretical observations. 
The other one is to keep Fourier coefficients with large absolute
values, which offers us a little bit extra flexibility, especially
when the whole graphical model is dominated by a few key variables and
we would like to go over the degree limitations occasionally.
%
We found both variants work equally well.

%% file: experiments.tex
\section{Experiments}
\begin{figure}[!t]
  \centering
  \includegraphics[width=0.49\linewidth]{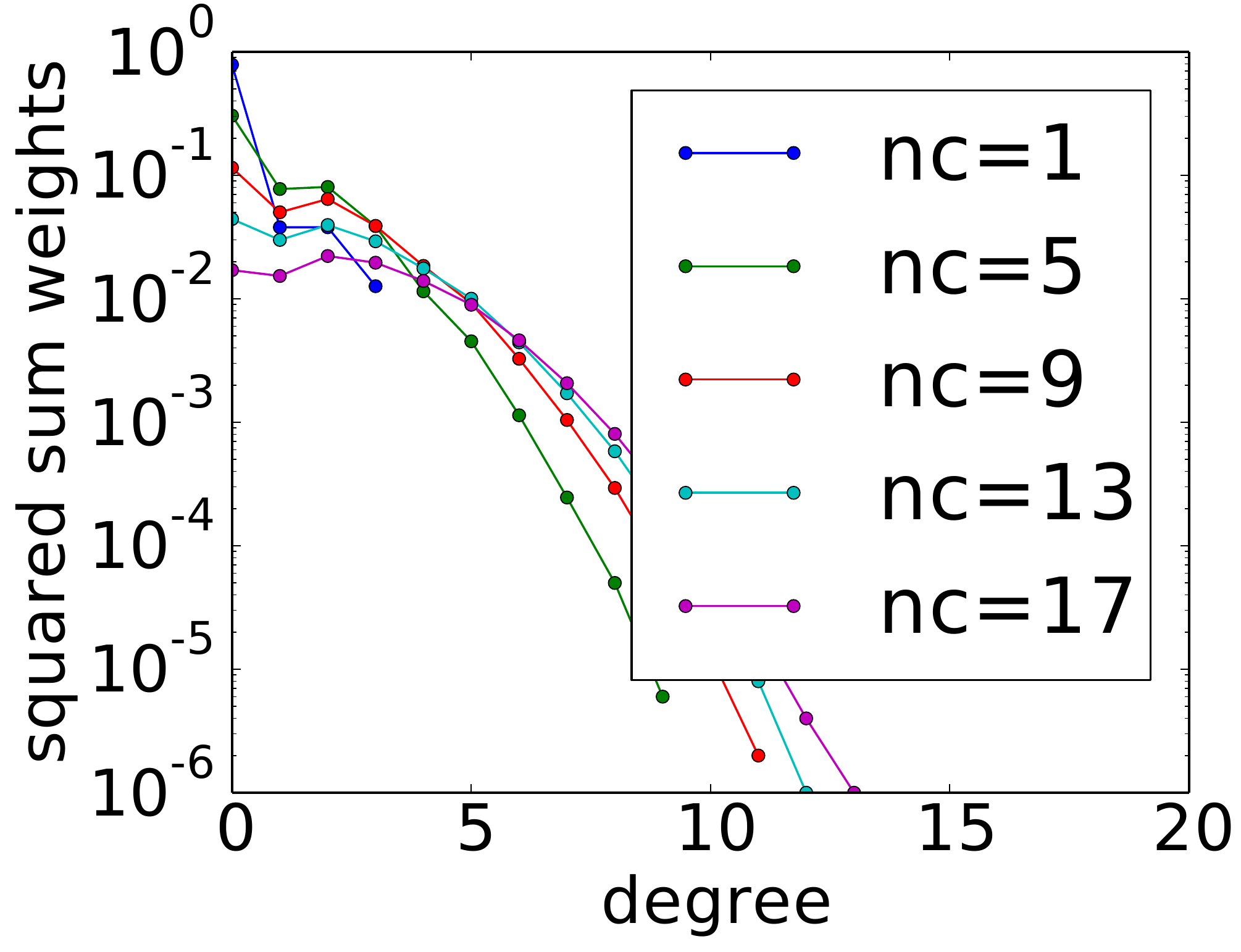}
  \includegraphics[width=0.49\linewidth]{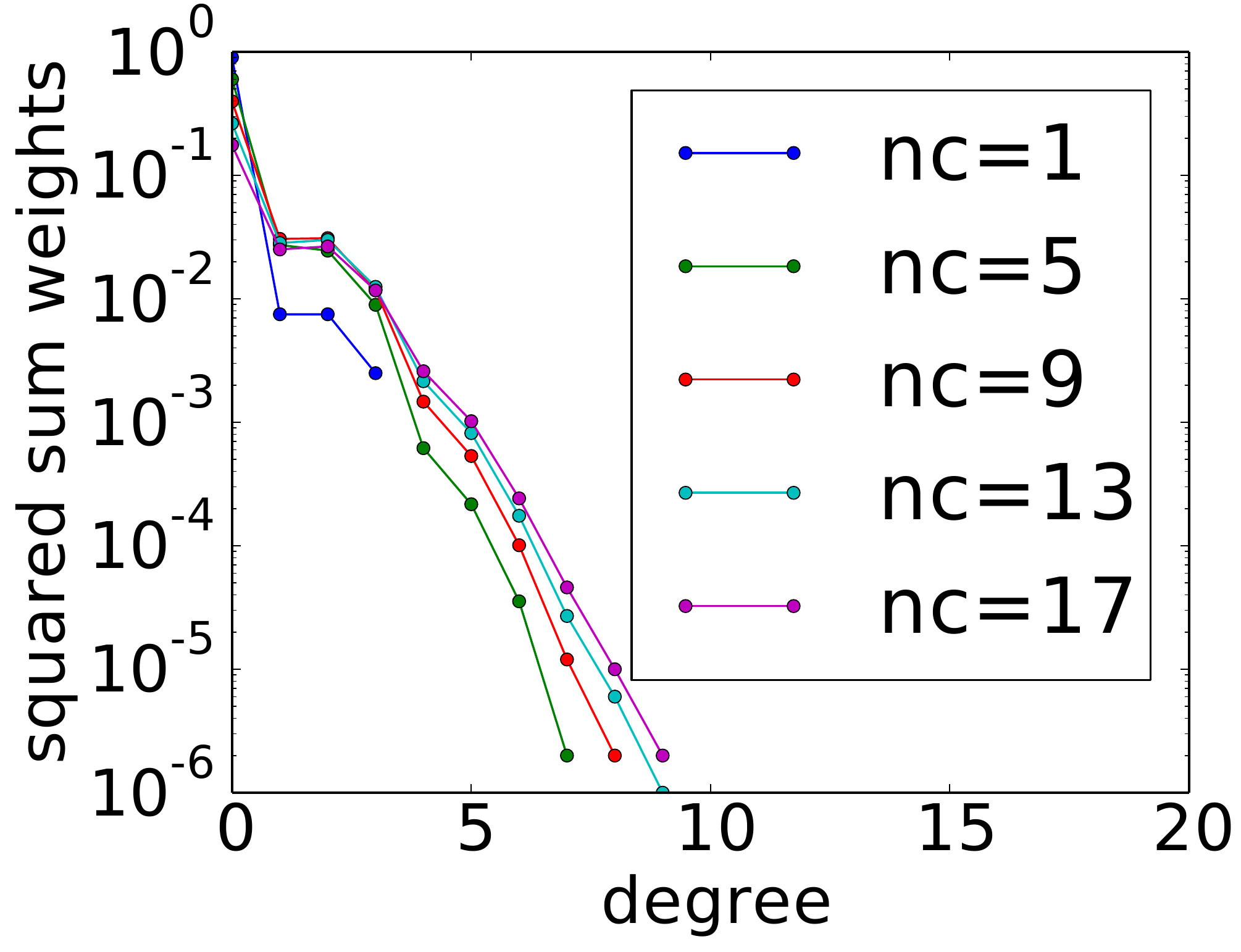}
  \caption{Weight concentration on low degree coefficients in the
    Fourier domain. Weight random 3-SAT instances, with 20 variables
    and nc clauses (Left) $\eta=0.1$, (Right) $\eta=0.6$.}
  \label{fig:wconcentrate}
  \vskip -0.5cm
\end{figure}
\begin{figure}[!t]
\centering
\begin{subfigure}[b]{0.49\linewidth}
\includegraphics[width=\textwidth]{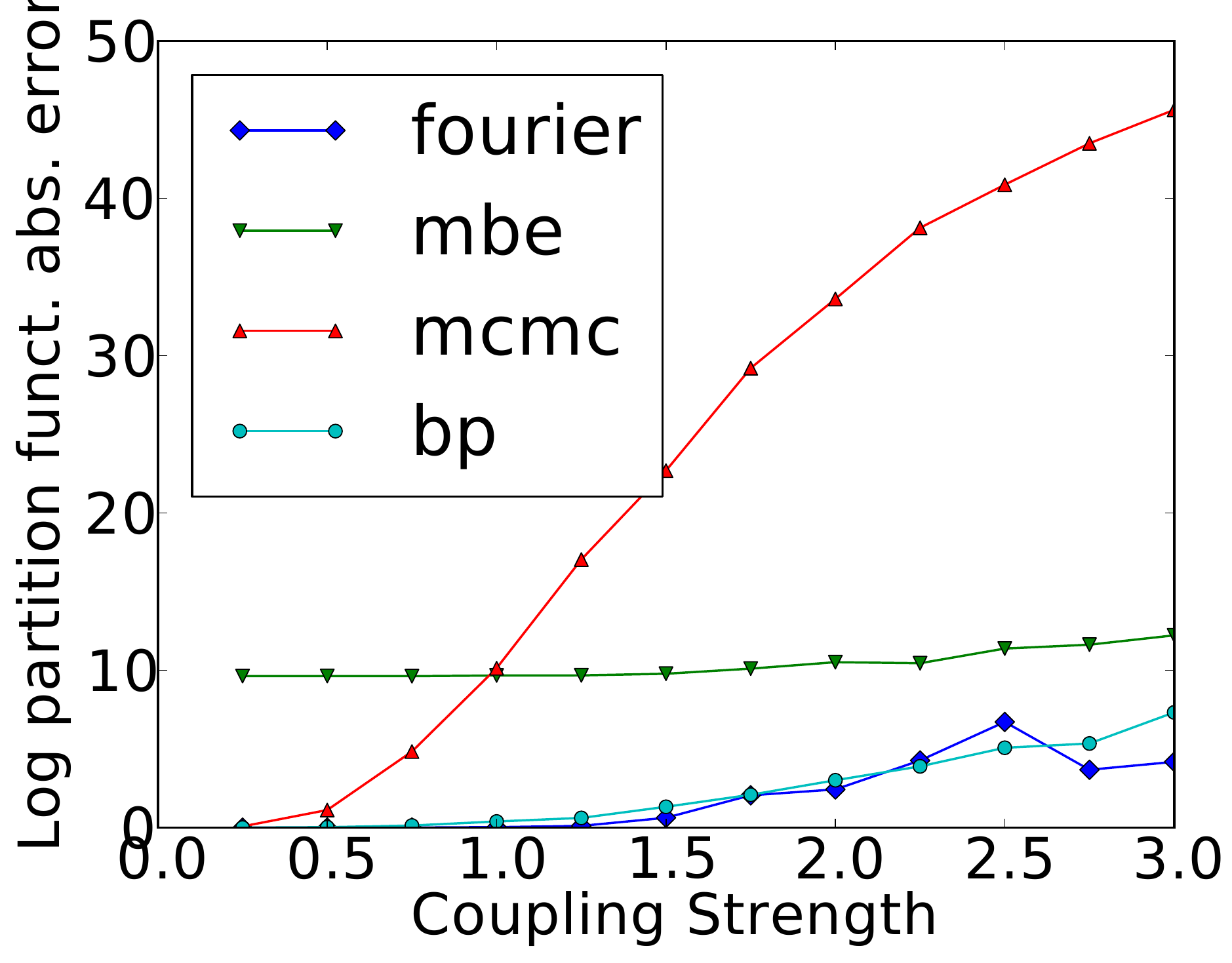}
\caption{Mixed. Field 0.01}
\label{fig:b}
\end{subfigure}
\begin{subfigure}[b]{0.49\linewidth}
\includegraphics[width=\textwidth]{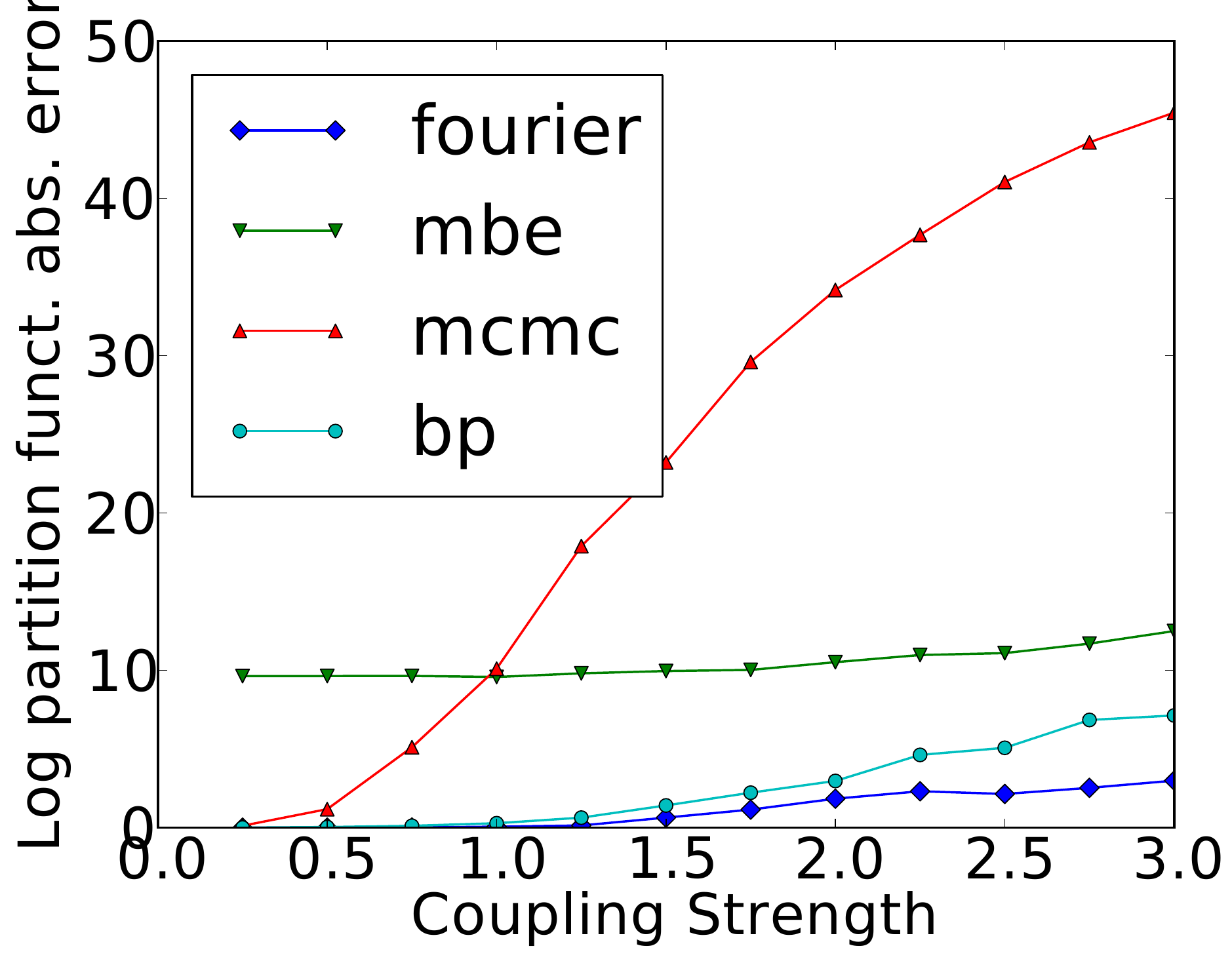}
\caption{Mixed. Field 0.1}
\label{fig:c}
\end{subfigure}
\caption{Log-partition function absolute errors for $15 \times 15$ small scale Ising Grids. Fourier is for the VE Algorithm in the Fourier domain. mbe is for Mini-bucket Elimination. BP is for Belief Propagation. Large scale experiments are on the next page.}\label{fig:ising_res}
\vskip -0.5cm
\end{figure}

\subsection{Weight Concentration on Low Degree Coefficients}
We first validate our theoretical results on the weight concentration
on low-degree coefficients in Fourier representations. We 
evaluate our results on random weighted 3-SAT instances with 20
variables.
Small instances are chosen because we have to compute the full Fourier
spectrum.
The weighted 3-SAT instances is specified by a CNF and a weight
$\eta$. Each factor corresponds to a clause in the CNF. When the
clause is satisfied, the corresponding factor evaluates to 1,
otherwise evaluates to $\eta$.
For each $\eta$ and the number of clauses $nc$, we randomly generate 100
instances. For each instance, we compute the squared sum weight at
each degree: $\mathcal{W}_{k}[f] = \sum_{S \subseteq [n], |S| = k}
\hat{f}(S)^2$.
Figure~\ref{fig:wconcentrate} shows the median value of the squared
sum weight over 100 instances for given $\eta$ and $nc$ in log scale.
As seen from the figure, although the full representation involves
coefficients up to degree 20 (20 variables), the weights are
concentrated on low degree coefficients (up to 5), regardless of
$\eta$, which is in line with the theoretical result.

\subsection{Applying Fourier Representation in Variable Elimination}

\addtolength{\belowcaptionskip}{-10pt}
\begin{table*}[!tb]
  \centering
  \begin{tabular}{c|c|cccccc}
    \hline
    Category & \#ins & Minibucket & Fourier (max coef) & Fourier (min deg) & BP & MCMC & HAK\\
    \hline
    bn2o-30-* & 18 & 3.91 & $1.21 \cdot 10^{-2}$ & $1.36 \cdot 10^{-2}$ & $0.94 \cdot 10^{-2}$ & 0.34 & ${\bf 8.3\cdot 10^{-4}}$\\ 
    grids2/50-* & 72 & 5.12 & ${\bf 3.67 \cdot 10^{-6}}$ & $7.81 \cdot 10^{-6}$ & $1.53 \cdot 10^{-2}$ & -- & $1.53 \cdot 10^{-2}$\\
    grids2/75-* & 103 & 18.34 & ${\bf 5.41  \cdot 10^{-4}}$ \ & $6.87 \cdot 10^{-4}$ & $2.94   \cdot 10^{-2}$ & -- & $2.94   \cdot 10^{-2}$\\
    grids2/90-* & 105 & 26.16 & ${\bf 2.23 \cdot 10^{-3}}$ & $5.71 \cdot 10^{-3}$ & $5.59 \cdot 10^{-2}$ & -- & $5.22 \cdot 10^{-2}$\\
    blockmap\_05* & 48 & $1.25 \cdot 10^{-6}$ & ${\bf 4.34 \cdot 10^{-9}}$ & ${\bf 4.34 \cdot 10^{-9}}$ & 0.11 & -- & $8.73\cdot 10^{-9}$\\
    students\_03* & 16 & $2.85\cdot 10^{-6}$ & ${\bf 1.67\cdot 10^{-7}}$ & ${\bf 1.67\cdot 10^{-7}}$ & 2.20 & -- & $3.17\cdot 10^{-6}$\\
    mastermind\_03* & 48 & 7.83 & 0.47 & 0.36 & 27.69 & -- & ${\bf 4.35\cdot 10^{-5}}$\\
    mastermind\_04* & 32 & 12.30 & ${\bf 3.63 \cdot 10^{-7}}$ & ${\bf 3.63 \cdot 10^{-7}}$ & 20.59 & -- & $4.03\cdot 10^{-5}$\\
    mastermind\_05* & 16 & 4.06 & ${\bf 2.56\cdot 10^{-7}}$ & ${\bf 2.56\cdot 10^{-7}}$ & 22.47 & -- & $3.02\cdot 10^{-5}$\\
    mastermind\_06* & 16 & 22.34 & ${\bf 3.89\cdot 10^{-7}}$ & ${\bf 3.89\cdot 10^{-7}}$ & 17.18 & -- & $4.5\cdot 10^{-5}$\\
    mastermind\_10* & 16 & 275.82 & 5.63 & 2.98 & 26.32 & --& ${\bf 0.14}$\\
    \hline
  \end{tabular}
  \caption{The comparsion of various inference algorithms on several categories in UAI 2010 Inference Challenge. The median differences in log partition function $|\log_{10} Z_{\mbox{approx}} - \log_{10} Z_{\mbox{true}}|$ averaged over benchmarks in each category are shown. Fourier VE algorithms outperform Belief Propagation, MCMC and Minibucket Algorithm. \#ins is the number of instances in each category.} 
  \label{tab:uai2010}
  \vskip -0.3cm
\end{table*}
\addtolength{\belowcaptionskip}{10pt}
%
We integrate the Fourier representation into the variable elimination
algorithm, and evaluate its performance as an approximate probabilistic
inference scheme to estimate the partition function of undirected graphical models.
We implemented two versions of the Fourier Variable Elimination
Algorithm. One version always keeps coefficients with the largest
absolute values when we truncate the representation. The other
version keeps coefficients with the lowest degree.
Our main comparison is against Mini-Bucket Elimination, since the two
algorithms are both based on variable elimination, with the only
difference being the way in which the messages are approximated.
We obtained the source code from the author of Mini-Bucket
Elimination, which includes sophisticated heuristics for splitting
factors.
The versions we obtained are used for Maximum A Posteriori Estimation
(MAP). We augment this version to compute the partition function by
replacing the maximization operators by summation operators.
We also compare our VE algorithm with MCMC and Loopy Belief
Propagation. We implemented the classical Ogata-Tanemura scheme
\cite{OgataTanemura1981} with Gibbs transitions in MCMC to estimate the partition function.
We use the implementation in LibDAI \cite{Mooij2010libDAI} for belief
propagation, with random updates, damping rate of 0.1 and the maximal
number of iterations 1,000,000.
Throughout the experiment, we control the number of MCMC steps, the
$i$-bound of Minibucket and the message size of Fourier VE to make
sure that the algorithms complete in reasonable time (several
minutes).

We first compare on small instances for which we can compute ground truth using  the state-of-the-art exact inference
algorithm ACE \cite{Darwiche2002compilation}.
We run on 15-by-15 Ising models with mixed coupling
strengths and various field strengths. We run 20 instances for each
coupling strength.
For a fair comparison, we fix the size of the messages for both
Fourier VE and Mini-bucket to $2^{10}=1,024$. Under this message size
VE algorithms cannot handle the instances exactly.
Figure~\ref{fig:ising_res} shows the results. The performance of the
two versions of the Fourier VE algorithm are almost the same, so we only
show one curve. Clearly the Fourier VE Algorithm outperforms
the MCMC and the Mini-bucket Elimination. It also outperforms 
Belief Propagation when the field strength is relatively strong.

In addition, we compare our inference algorithms on large benchmarks
from the UAI 2010 Approximate Inference Challenge
\cite{UAI10Challenge}.
Because we need the ground truth to compare with, we only consider
benchmarks that can be solved by ACE \cite{Darwiche2002compilation} in
2 hours time, and 8GB of memory.
The second column of Table~\ref{tab:uai2010} shows the number of
instances that ACE completes with the exact answer. 
The 3rd to the 7th column of Table~\ref{tab:uai2010} shows the result for several inference algorithms, including the Minibucket algorithm with $i$-bound of 20, two versions of the Fourier Variable Elimination algorithms, belief propagation and MCMC. To be fair with Minibucket, we set the message size for Fourier VE to be 1,048,576 ($2^{20}$). Because the
complexity of the multiplication step in Fourier VE is quadratic in
the number of coefficients, we further shrink the message size to
1,024 ($2^{10}$) during multiplication.
We allow 1,000,000 steps for burn in and another 1,000,000 steps for
sampling in the MCMC approach.
The same with the inference challenge, we compare inference algorithms
on the difference in the log partition function $|\log
Z_{\mbox{approx}} - \log Z_{\mbox{true}}|$.
The table reports the median differences, which are averaged over all benchmarks in each category.
%
%
If one algorithm fails to complete on one instance,
we count the difference in partition function as $+\infty$, so it is
counted as the worst case when computing the median.
For MCMC, ``--'' means that the Ogata-Tanemura scheme did not find a
belief state with substantial probability mass, so the result is way
off when taking the logarithm.
%
%
The results in Table~\ref{tab:uai2010} show that Fourier Variable
Elimination algorithms outperform MCMC, BP and Minibucket on
many categories in the Inference challenge. 
In particular, Fourier VE works well on grid and structural instances.
We also listed the performance of a Double-loop Generalized Belief
Propagation~\cite{Heskes2002HAK} in the last column of
Table~\ref{tab:uai2010}.
This implementation won one category in the Inference challenge, and
contains various improvements besides the techniques presented
in the paper. 
%
We used the parameter settings for high precision in the
Inference challenge for HAK.
As we can see, Fourier VE matches or outperforms this
implementation in some categories.
Unlike fully optimized HAK, Fourier VE is a simple variable
elimination algorithm, which involves passing messages only once.
Indeed, the median time for Fourier VE to complete on bn2o
instances is about 40 seconds, while HAK takes 1800 seconds.
We are researching on incorporating the Fourier representation into
message passing algorithms.


\begin{figure}[tb]
\centering
\begin{subfigure}[b]{0.48\linewidth}
\includegraphics[width=\textwidth]{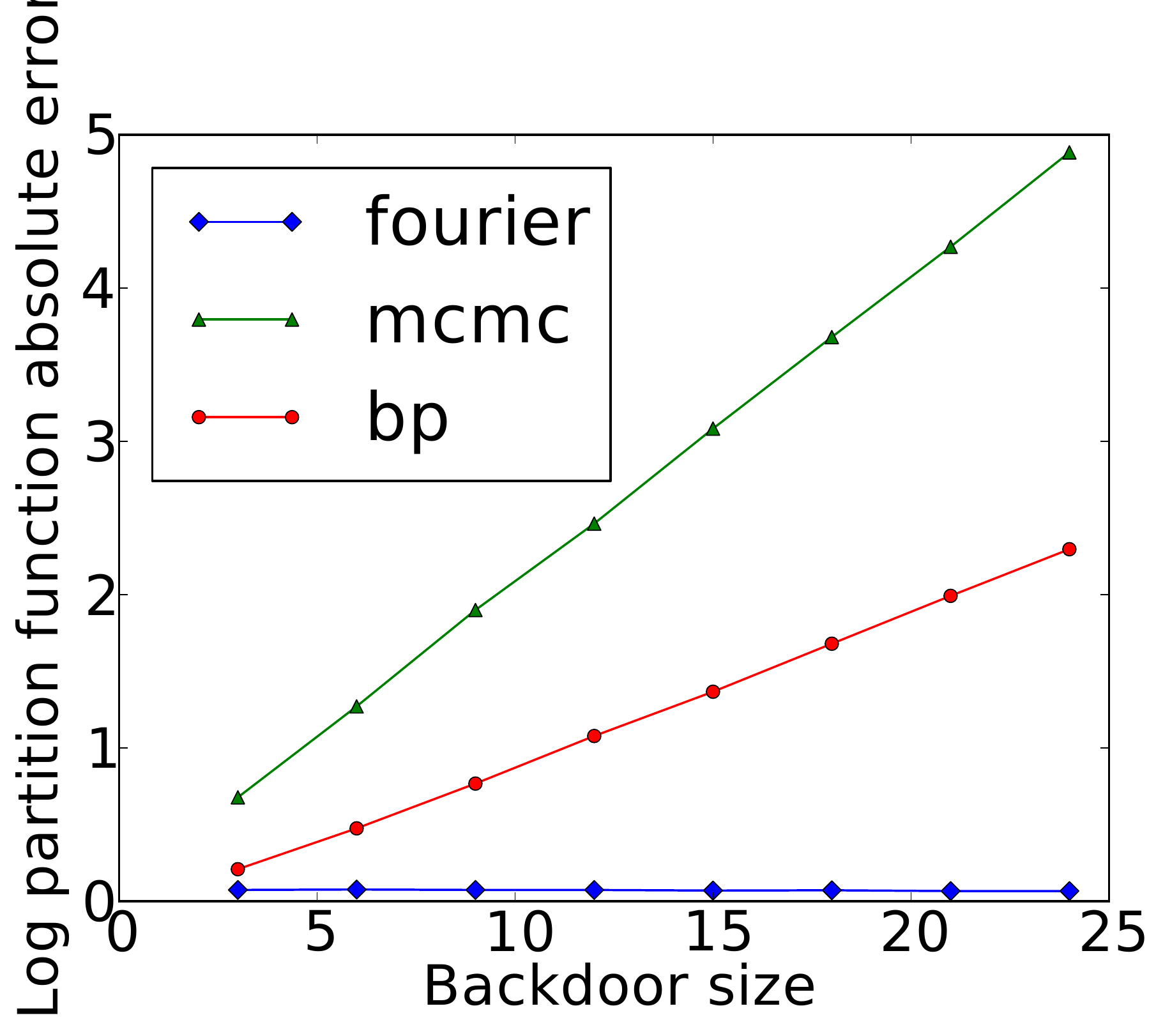}
\caption{Independent backdoors}
\label{fig:a2}
\end{subfigure}
~ 
\begin{subfigure}[b]{0.48\linewidth}
\includegraphics[width=\textwidth]{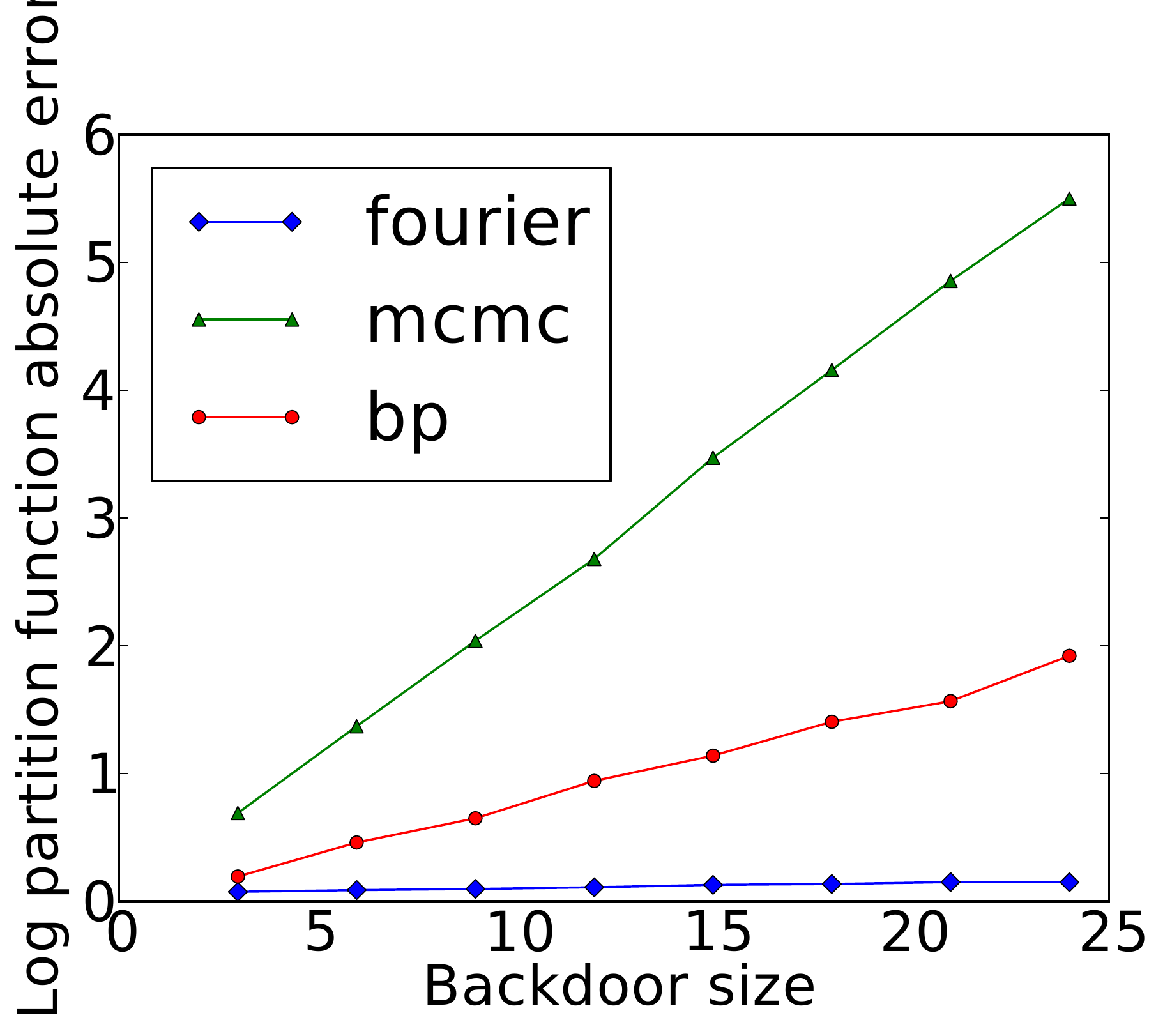}
\caption{Linked backdoors}
\label{fig:b2}
\end{subfigure}
\caption{Log-partition function absolute errors for Weighted Models with Backdoor Structure.}
\label{fig:bckdoor}
\vskip -0.5cm
\end{figure}
Next we evaluate their performance on a synthetically generated
benchmark beyond the capability of exact inference algorithms. For one
instance of this benchmark, we randomly generate factors of size $3$
with low coupling weights. We then add a backdoor structure to each
instance, by enforcing coupling factors of size $3$ in which the $3$
variables of the factor must take the same value. For these instances,
we can compute the expected value of the partition function and
compare it with the output of the algorithms. We report the results on
Figure \ref{fig:bckdoor}. Here the experimental setup for each
inference algorithm is kept the same as the previous algorithm. The
Mini-bucket approach is not reported, as it performs very poorly on
these instances. The performance of the two implementations of Fourier VE are again similar, so they are combined into one curve. 
These results show that the Fourier approach
outperforms both MCMC and Belief Propagation, and suggest that it can
perform arbitrarily better than both approaches as the size of the
backdoor increases.

%
%


Finally, we compare different inference algorithms on a machine
learning application. Here we learn a grid Ising model from data. 
The computation of the partition function is beyond any exact
inference methods.
Hence in order to compare the performance of different inference
algorithms, we have to control the training data that are fit into the
Ising Model, to be able to predict what the learned model looks like.
To generate training pictures, we start with a template with nine
boxes (shown in Figure~\ref{fig:ml_template}).
The training pictures are of size $25\times 25$, so the partition function cannot be computed exactly by variable elimination algorithms with message size $2^{20} = 1,048,576$.  
Each of the nine boxes in the template will have a 50\% opportunity to
appear in a training picture, and the occurrences of the nine
boxes are independent of each other.
We further blur the training images with 5\% white noise. 
Figures~\ref{fig:ml_train_1} and \ref{fig:ml_train_2} show two examples of
the generated training images.
We then use these training images to learn a grid Ising Model:
\begin{align*}
  Pr(\x) &= \frac{1}{Z} \exp \left(\sum_{i\in V} a_{i} x_{i} + \sum_{(i,j)\in E} b_{i,j}  x_{i} x_{j} \right),
\end{align*}
where $V$, $E$ are the node and edge set of a grid, respectively. We
train the model using contrastive divergence
\cite{Hinton2002ContrasiveDivergence}, with $k=15$ steps of blocked
Gibbs updates, on $20,000$ such training images. (As we will see,
\emph{vanilla} Gibbs sampling, which updates one pixel at a time, does
not work well on this problem.)
We further encourage a sparse model by using a L1 regularizer. 
\begin{figure}[tb]
  \centering
  \begin{subfigure}[b]{0.24\linewidth}
    \includegraphics[width=\textwidth]{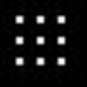}
    \caption{Template}
    \label{fig:ml_template}
  \end{subfigure}
  \begin{subfigure}[b]{0.24\linewidth}
    \includegraphics[width=\textwidth]{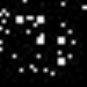}
    \caption{Train Pic 1}
    \label{fig:ml_train_1}
  \end{subfigure}
  \begin{subfigure}[b]{0.24\linewidth}
    \includegraphics[width=\textwidth]{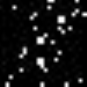}
    \caption{Train Pic 2}
    \label{fig:ml_train_2}
  \end{subfigure}
  \\
  \begin{subfigure}[b]{0.24\linewidth}
    \includegraphics[width=\textwidth]{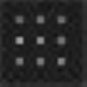}
    \caption{Fourier}
    \label{fig:ml_fourier}
  \end{subfigure}
  \begin{subfigure}[b]{0.24\linewidth}
    \includegraphics[width=\textwidth]{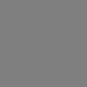}
    \caption{MCMC}
    \label{fig:ml_gibbs}
  \end{subfigure}  
  \begin{subfigure}[b]{0.24\linewidth}
    \includegraphics[width=\textwidth]{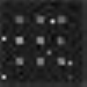}
    \caption{mbe}
    \label{fig:ml_mbe}
  \end{subfigure}
  \begin{subfigure}[b]{0.24\linewidth}
    \includegraphics[width=\textwidth]{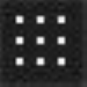}
    \caption{Mean Field}
    \label{fig:ml_mf}
  \end{subfigure}
 \vspace*{-0.15 in}
 \caption{Comparison of several inference algorithms on computing the marginal probabilities of an Ising model learned from synthetic data. (a) The template to generate training images and (b,c) two example images in the training set. (d,e,f,g) The marginal probabilities obtained via four inference algorithms. Only the Fourier algorithm captures the fact that the 9 boxes are presented half of the time independently in the training data.}
  \label{fig:ml_exp}
  \vskip -0.5cm
\end{figure}
Once the model is learned, we use inference algorithms to compute the
marginal probability of each pixel. 
Figure~\ref{fig:ml_fourier}~\ref{fig:ml_gibbs}~\ref{fig:ml_mbe} and
\ref{fig:ml_mf} show the marginals computed for the Fourier VE, MCMC,
Minibucket Elimination, and the Mean Field on the learned model (white
means the probability is close to 1, black means close to 0).
Both the Minibucket and the Fourier VE keep a message size of
$2^{20}=1,048,576$, so they cannot compute the marginals
exactly. Fourier VE keeps coefficients with largest absolute value during multiplication.
For pixels outside of the nine boxes, in most circumstances they are
black in the training images. Therefore, their marginals in the
learned model should be close to 0.
For pixels within the nine boxes, half of the time they are white in
the training images. 
%
Hence, the marginal probabilities of these pixels in the learned model
should be roughly 0.5. 
We validated the two aforementioned empirical observations on
images with small size which we can compute the marginals exactly.
As we can see, only the Fourier Variable Elimination Algorithm is able
to predict a marginal close to 0.5 on these pixels. The performance of the MCMC algorithm
(a Gibbs sampler, updating one pixel at a time) is poor.
The Minibucket Algorithm has noise on some
pixels. The marginals of the nine boxes predicted by mean field are
close to 1, a clearly wrong answer.

%% file: conclusion.tex
\section{Conclusion}

We explore a novel way to exploit compact representations of
high-dimensional probability distributions in approximate
probabilistic inference.
Our approach is based on discrete Fourier Representation of weighted
Boolean Functions, complementing the classical method of exploiting 
conditional independence between the variables.
We show that a large class of weighted probabilistic graphical models
have a compact Fourier representation. This theoretical result opens
up a novel way of approximating probability distributions.
We demonstrate the significance of this approach by applying it to the
variable elimination algorithm, obtaining very encouraging results.

\section*{Acknowledgments}
This research was supported by National Science Foundation (Award
Number 0832782, 1522054, 1059284).

%% file: theory_supply.tex
\section*{Supplementary Materials}

\subsection*{Proof of Lemma 1}

Let $R_n^l$ be the collection of restrictions on $n$ Boolean variables
$x_1, \ldots, x_n$. Each restriction in $R_n^l$ leaves a set of $l$ variables $J=\{x_{i_1}, \ldots, x_{i_l}\}$ open, while it fixes all other
variables $x_i \not \in J$ to either -1 or 1. It is easy to see
that the size of $R_n^l$ is given by:
\begin{equation}
  |R_n^l| = {n \choose l} \cdot 2^{n-l}.
\end{equation}

For a restriction $J|\z \in R_n^l$, call $J|\z$ \textit{bad} if and
only if for all decision tree $h$ with depth $t$, there exists at
least one input $\x_J$, such that $|h(\x_J) - f|_{J|\z}(\x_J)| >
\gamma$. Let $B_n^l$ be the set of all bad restrictions, ie: 
$B_n^l = \{J|\z \in R_n^l: J|\z \mbox{ is bad}\}$. 
To prove the lemma, it is sufficient to prove that 
\begin{equation}
  \frac{|B_n^l|}{|R_n^l|} \leq \frac{1}{2} \left(\frac{l}{n-l} 8 u w\right)^t. 
\end{equation}

In the proof that follows, for every bad restriction $\rho \in B_n^l$,
we establish a bijection between $\rho$ and $(\xi, s)$, in which $\xi$
is a restriction in $R_n^{l-t}$ and $s$ is a certificate from a
witness set $A$. In this case, the number of distinct $\rho$'s is
bounded by the number of $(\xi, s)$ pairs: 
\begin{equation}
  |B_n^l| \leq |R_n^{l-t}| \cdot |A|.
\end{equation}

For a restriction $\rho$, we form the \textit{canonical decision
 tree} for $f|_\rho$ under precision $\gamma$ as follows: 
\begin{enumerate}
\item We start with a fixed order for the
  variables and another fixed order for the factors. 
\item If $f|_\rho$ is a constant function, or $||f|_\rho||_\infty \leq
  \gamma$, stop.
\item Otherwise, under restriction $\rho$, some factors evaluate to
  fixed values (all variables in these factors are fixed or there are
  free variables, but all assignments to these free variables lead to
  value 1), while other factors do not.  Examine the factors according
  to the fixed factor order until reaching the first factor that still
  does not evaluate to a fixed value.
\item Expand the open variables of this factor, under the fixed
  variable order specified in step 1. The result will be a tree (The
  root branch is for the first open variable. The branches in the
  next level is for the second open variable, etc).
\item Each leaf of this tree corresponds to $f|_{\rho\pi_1}$, in
  which $\pi_1$ is a value restriction for all open variables of the
  factor. Recursively apply step 2 to 5 for function
  $f|_{\rho\pi_1}$, until the condition in step 2 holds. Then attach
  the resulting tree to this leaf.
\end{enumerate}

Figure~\ref{fig:decision_tree} provides a graphical demonstration of a
canonical decision tree.

Now suppose restriction $\rho$ is bad. By definition, for any decision
tree of depth $t$, there exists at least one input $\x$, such that
$|h(\x) - f|_{\rho}(\x)| > \gamma$.
The canonical decision tree is no exception. Therefore, there must be
a path $l$ in the canonical decision tree of $f|_\rho$, which has more
than $t$ variables. Furthermore, these $t$ variables can be split into
$k$ ($1\leq k \leq t$) segments, each of which corresponds to one
factor.
Let $f_i$ ($i\in\{1,\dots,k\}$) be these factors, and let $\pi_i$ be
the assignments of the free variables for $f_i$ in path $l$. 
Now for each factor $f_i$, by the definition of the canonical decision
tree, under the restriction $\rho\pi_1\dots\pi_{i-1}$,
$f_i|\rho\pi_1\dots\pi_{i-1}$ must have a branch whose value is no
greater than $\eta$ (otherwise $f_i|\rho\pi_1\dots\pi_{i-1}$ all
evaluates to 1). We call this branch the ``compressing'' branch for factor $f_i|\rho\pi_1\dots\pi_{i-1}$. 
Let the variable assignment which leads to this compressing branch for
$f_i|\rho\pi_1\dots\pi_{i-1}$ be $\sigma_i$. Let $\sigma = \sigma_1
\ldots \sigma_k$. Then we map the bad restriction $\rho$ to $\rho
\sigma$ and an auxiliary advice string that we are going to describe.
\begin{figure}
  \centering
  \includegraphics[width=0.9\linewidth]{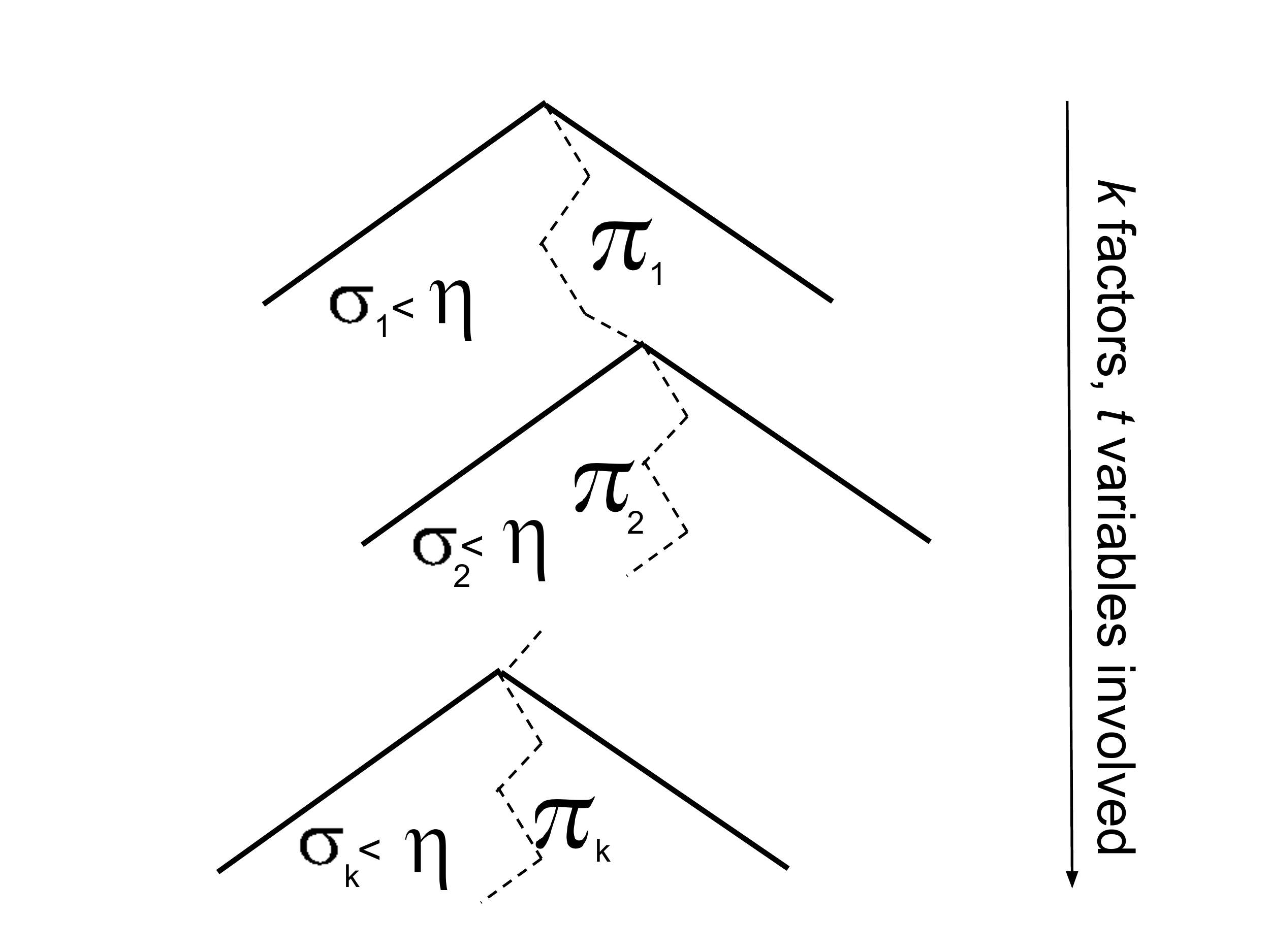}
  \caption{A graphical illustration of a canonical decision tree.}
  \label{fig:decision_tree}
\end{figure}

It is self-explanatory that we can map from any bad restriction $\rho$
to $\rho \sigma$. The auxiliary advice is used to establish the
backward mapping, i.e. the mapping from $\rho\sigma$ to $\rho$.
When we look at the result of $f|\rho\sigma$, we will notice that at
least one factor is set to its compressing branch (because we set
$f_1$ to its compressing branch in the forward mapping). 
Now there could be other factors set at their compressing branches
(because of $\rho$), but an important observation is that: \textit{the
  number of factors at their compressing branches cannot exceed
  $u=\lceil \log_\eta \gamma \rceil+1$}, because otherwise, the other
$u-1$ factors already render $||f|\rho||_\infty \leq \gamma$, and the
canonical decision tree should have stopped on expanding this branch.
We therefore could record the index number of $f_1$ out of all the
factors that are fixed at their compressing branches in the auxiliary
advice string, so we can find $f_1$ in the backward mapping. Notice
that this index number will be between 1 and $u$, so it takes $\log u$
bits to store it.

Now with the auxiliary information, we can identify which factor is
$f_1$. The next task is to identify which variables in $f_1$ are fixed
by $\rho$, and which are fixed by $\sigma_1$. Moreover, if one
variable is fixed by $\sigma_1$, we would like to know its correct
values in $\pi_1$. To do this, we introduce additional auxiliary
information: for each factor $f_i$, suppose it has $r_i$ free
variables under restriction $f_i|\rho\pi_1\dots\pi_{i-1}$, we use
$r_i$ integers to mark the indices of these free variables. Because each $f_i$ is of width at most $w$, every integer of this type is between 1 and $w$ (therefore can be stored in $\log w$ bits). Also, it requires $t$ integers of this type in total to keep this
information, because we have $t$ free variables in total for $f_1,
\ldots, f_k$.
%

Notice that it is not sufficient to keep these integers. We further
need $k-1$ separators, which tell which integer belongs to which factor 
$f_i$. Aligning these integers in a line, we need $k-1$ separators to
break the line into $k$ segments. These separators can be represented
by $t-1$ bits, in which the $i$-th bit is 1 if and only if there is a
separator between the $i$-th and ($i$+1)-th integer (we have $t$
integers at most).
With these two pieces of information, we are able to know the locations of free variables set by $\sigma_i$ for each factor $f_i$.

We further need to know the values for each variable in
$\pi_i$. Therefore, we add in another $t$-bit string, each bit is
either 0 or 1. 0 means the assignment of the corresponding variable in
$\pi_i$ is the same as the one in $\sigma_i$, 1 means the opposite.

With all this auxiliary information, we can start from $\rho \sigma$,
find the first factor $f_1$, further identify which variables are set
by $\sigma_1$ in $f_1$, and set back its values in $\pi_1$. Then we
start with $f|\pi_1$, we can find $\pi_2$ in the same process, and
continue. Finally, we will find all variables in $\sigma$ and back up
the original restriction $\rho$.

Now to count the length of the auxiliary information, the total length
is $t\log u + t\log w + 2t - 1$ bits. Therefore, we can have a
one-to-one mapping between elements in $B_n^l$ and $R_n^{l-t} \times
A$, in which the size of $A$ is bounded by $2^{t\log u + t\log w + 2t
  - 1} = (uw)^t \cdot 2^{2t-1}$.

In all, 
\begin{align}
  \frac{|B_n^l|}{|R_n^l|} &\leq \frac{{n \choose l-t} 2^{n-l+t} (uw)^t \cdot 2^{2t-1}}{{n \choose l} 2^{n-l}}\\
  &=\frac{{n \choose l-t} \frac{1}{2} (8uw)^t }{{n \choose l} }\\
  &=\frac{l(l-1)\ldots(l-t+1)}{(n-l+1)\ldots(n-l+t)} \frac{1}{2}(8uw)^t\\
  &\leq \frac{1}{2} \left(\frac{l}{n-l} 8 u w\right)^t. 
\end{align}

\subsection*{Proof of Theorem 3}

For each term in the Fourier expansion whose degree is less than or equal to $d$, we can treat this term as a weighted function involving less than or equal to $d$ variables. Therefore, it can be represented by a decision tree, in which each path of the tree involves no more than $d$ variables (therefore the tree is at most at the depth of $d$). Because $f$ is represented as the sum over a set of Fourier terms up to degree $d$, it can be also represented as the sum of the corresponding decision trees. 

\subsection*{Proof of Theorem 6}

Let the Fourier expansion of $f$ be: $f(\x)= \sum_S \hat{f}(S) \chi_S(\x)$, we have: 
\begin{align*}
&f'(\x \setminus x_i)\\
=&f(\x \setminus x_i, x_i=+1) + f(\x \setminus x_i, x_i=-1)\\
=&\sum_{S: i \in S} \hat{f}(S) \cdot \chi_{S\setminus i}(\x \setminus x_i) \cdot 1+\\
 &\sum_{S: i \not \in S} \hat{f}(S) \cdot \chi_{S\setminus i}(\x \setminus x_i)  \\
&  + \sum_{S: i \in S} \hat{f}(S) \cdot \chi_{S\setminus i}(\x \setminus x_i) \cdot (-1) + \\
 &\sum_{S: i \not \in S} \hat{f}(S) \cdot \chi_{S\setminus i}(\x \setminus x_i)\\
=&\sum_{S: i \in S} \hat{f}(S) \cdot \chi_{S\setminus i}(\x \setminus x_i) \cdot 1 + \\
&\sum_{S: i \in S} \hat{f}(S) \cdot \chi_{S\setminus i}(\x \setminus x_i) \cdot (-1)  \\
& + \sum_{S: i \not \in S} \hat{f}(S) \cdot \chi_{S\setminus i}(\x \setminus x_i) + \\
&\sum_{S: i \not \in S} \hat{f}(S) \cdot \chi_{S\setminus i}(\x \setminus x_i)\\
=& \sum_{S: i \not \in S} 2 \cdot \hat{f}(S) \cdot \chi_{S\setminus i}(\x \setminus x_i).
\end{align*}